\documentclass{article}
\usepackage{ijcai23}

\usepackage{times}
\usepackage{soul}
\usepackage{url}
\usepackage[hidelinks]{hyperref}
\usepackage[utf8]{inputenc}
\usepackage[small]{caption}
\usepackage{graphicx}
\usepackage{subcaption}
\usepackage{amsthm}
\usepackage{amssymb}
\usepackage{amsmath}
\usepackage{booktabs}
\usepackage{algorithm}
\usepackage{algorithmic}
\usepackage[switch]{lineno}

\linenumbers

\newtheorem{example}{Example}
\newtheorem{theorem}{Theorem}
\newtheorem{lemma}{Lemma}
\newtheorem{definition}{Definition}
\newtheorem{proposition}{Proposition}

\title{Simulating Petri nets with Boolean Matrix Logic Programming}

\author{
Lun Ai$^1$
\and
Stephen H. Muggleton$^1$ \and
Shishun Liang$^2$\and
Geoff S. Baldwin$^{2}$
\affiliations
$^1$Department of Computing, Imperial College London, London, UK\\
$^2$Department of Life Science, Imperial College London, UK.\\
\emails
\{lun.ai15, s.muggleton, shishun.liang20, g.baldwin\}@imperial.ac.uk,
}

\begin{document}
\raggedbottom
\nolinenumbers
\maketitle

\begin{abstract}
Recent attention to relational knowledge bases has sparked a demand for understanding how relations change between entities. Petri nets can represent knowledge structure and dynamically simulate interactions between entities, and thus they are well suited for achieving this goal. However, logic programs struggle to deal with extensive Petri nets due to the limitations of high-level symbol manipulations. To address this challenge, we introduce a novel approach called Boolean Matrix Logic Programming (BMLP), utilising boolean matrices as an alternative computation mechanism for Prolog to evaluate logic programs. Within this framework, we propose two novel BMLP algorithms for simulating a class of Petri nets known as elementary nets. This is done by transforming elementary nets into logically equivalent datalog programs. We demonstrate empirically that BMLP algorithms can evaluate these programs 40 times faster than tabled B-Prolog, SWI-Prolog, XSB-Prolog and Clingo. Our work enables the efficient simulation of elementary nets using Prolog, expanding the scope of analysis, learning and verification of complex systems with logic programming techniques. 
\end{abstract}
\section{Introduction}
\label{sec:intro}
Relational knowledge bases as a form of knowledge representation have received great interest recently  \cite{nickel_review_2016,wang_knowledge_2017,hogan_knowledge_2021,ji_survey_2022}. Relational knowledge bases are structured repositories of knowledge that encapsulate information about entities and their relationships.  A number of them including ConceptNet \cite{speer_conceptnet_2017}, DBpedia \cite{lehmann_dbpedia_2015} and YAGO \cite{rebele_yago_2016} contain millions of entities and relations \cite{nickel_review_2016}. Current research on knowledge bases primarily assumes static relations and entity states \cite{ji_survey_2022}. Dynamic modelling of relations and entities has great importance since relational knowledge is continuously evolving. \cite{kazemi_representation_2020}. 

One class of graph models, called Petri nets, excels for this purpose. A Petri net is a weighted directed bipartite graph that contains two types of nodes. An example of a Petri net representing transit flight routes is presented in Figure \ref{fig:flight_routes}. The information flow between nodes in a Petri net is modelled by transitions of black tokens. This mechanism enables Petri nets to represent dynamic activities of systems in areas such as Computer Science and Biology \cite{reisig_understanding_2013,sahu_advances_2021}. Petri nets and logic programs can complement each other in knowledge representation and dynamic analysis \cite{domenici_petri_1990}. The example in Figure \ref{fig:flight_routes} relates to reachability analysis which is a fundamental problem in Petri nets \cite{murata_petri_1989}. However, the volume of Petri nets can present a significant challenge to logic programming.  For instance, in the OpenFlights database (Figure \ref{fig:flight_routes} \cite{openflights}), finding flight routes between cities may take substantial computation or may not terminate if cycles exist. 


\begin{figure}[t]
    \centering
        \includegraphics[width=\linewidth]{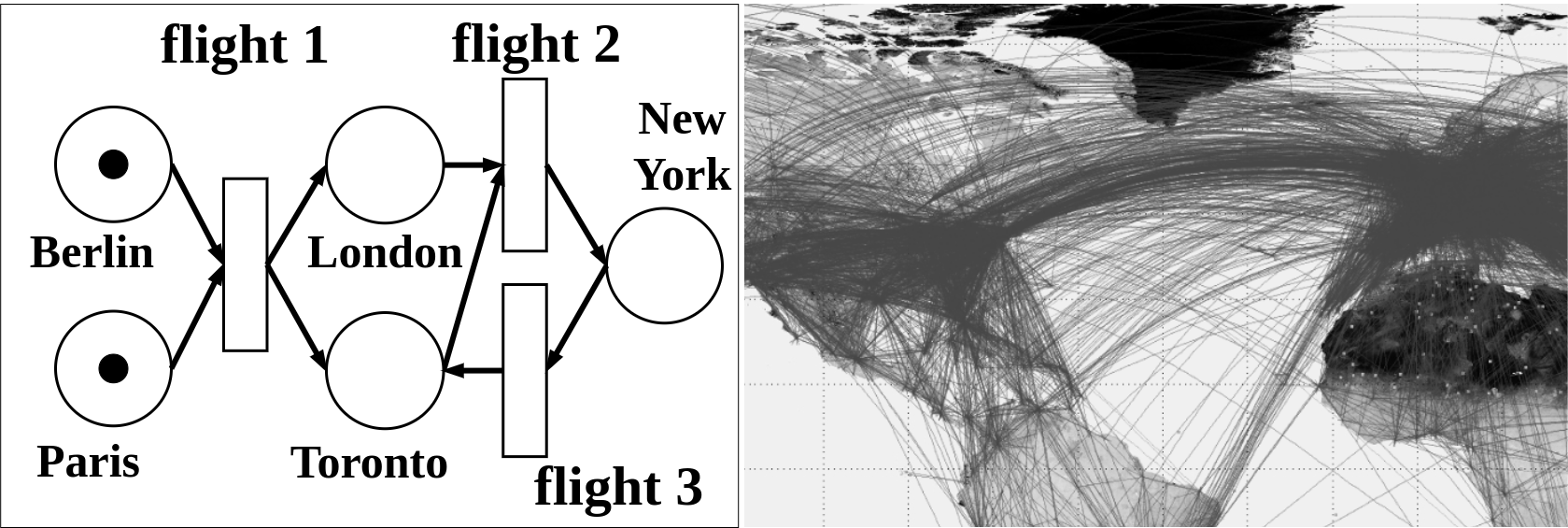} 
    \caption{Left: A Petri (elementary) net of transit flights between 5 cities with tokens on nodes ``Berlin'' and ``Paris''. Transit passengers can take different routes to their destinations. Passengers to New York from Berlin and Paris should arrive before the flights from London and Toronto go to New York. Right: Graphical view of relations in the OpenFlights database describing 67663 airplane flights between 3321 cities.}
    \vspace{-10pt}
    \label{fig:flight_routes}
\end{figure}

In this paper, we propose an approach called \textit{Boolean Matrix Logic Programming} (BMLP) to leverage the runtime efficiency advantages of low-level boolean matrices for computing logic programs. Our approach contrasts with traditional logical inference in AI, which has mostly been performed symbolically at a high level. We present two novel BMLP algorithms to analyse the reachability in a type of Petri nets called \textit{one-bounded elementary nets} (OENs). OENs have places that contain at most one token, allowing their dynamic behaviour to be expressible with boolean representations. We show OENs can be transformed into \textit{linear and immediately recursive} programs (such as program $\mathcal{P}_1$), characterised by a single recursive Horn clause. Experiments show the efficiency benefits of BMLP methods in simulating OENs with millions of nodes created from graphs and biological data.
\smallskip

\noindent \textbf{Novelty, impact and contributions:}
\begin{itemize}
    \item We define the Boolean Matrix Logic Programming (BMLP) problem. 
    \item We prove that computing reachability in an OEN is logically equivalent to evaluating a linear and immediately recursive program.
    \item We present a BMLP framework to evaluate linear and immediately recursive programs by compiling them into boolean matrices and computing the transitive closures.
    \item We introduce two algorithms, iterative extension (BMLP-IE) and repeated matrix squaring (BMLP-RMS) to compute the transitive closure of boolean matrices, with implementations in SWI-Prolog.
    \item Our empirical results show that BMLP-IE and BMLP-RMS can compute reachability in an OEN 40 times faster than recursive programs evaluated by Clingo \cite{gebser_clingo_2014} and Prolog systems with tabling (B-Prolog \cite{zhou_language_2012}, SWI-Prolog \cite{wielemaker:2011:tplp} and XSB-Prolog \cite{swift_xsb_2012}). 
\end{itemize}
\section{Related work}

\textbf{Logic programming}. Little attention has been paid to representing Petri nets with logic programs. The early work by Domenici~\shortcite{domenici_petri_1990} describes places and transitions of a Petri net in terms of n-ary ground facts. Srinivasan et al.~\shortcite{srinivasan_knowledge_2012} attempt to learn n-ary ground facts that represent Petri net transitions in an Inductive Logic Programming (ILP) \cite{ILP1991} system Aleph \cite{aleph}. Few studies have looked at encoding Petri nets by Answer Set Programming (ASP) \cite{gelfond_stable_1988}. Behrens and Dix~\shortcite{behrens_model_2007} use Petri nets for analysing interactions in multi-agent systems with ASP programs. Anwar et al.~\shortcite{anwar_encoding_2013} provide an ASP encoding of Petri nets with coloured tokens to model the behaviours of biological systems. This encoding is adapted by Dimopoulos et al.~\shortcite{dimopoulos_encoding_2020} for simulating reversible Petri nets that allow transitions to fire in both directions. These systems all use high-level symbolic representations of Petri nets as deductive knowledge bases. However, the runtime of deductive inference in large-scale knowledge bases limits the applicability of these systems.  

\textbf{Petri nets.} Petri nets have been used as abstract execution models of logic programs. The \textit{Predicate-Transition nets} (PrT-nets) \cite{genrich_predicatetransition_1981} is a class of Petri nets whose transitions are defined by logical implications and are annotated by the arguments in Horn clauses. The class of PrT-nets are logically equivalent to first-order Horn clauses \cite{murata_predicate_transition_1988}. Murata and Zhang~\shortcite{murata_predicate_transition_1988} combine PrT-nets and parallel processing to implement deductive forward proofs. Jeffrey et al.~\shortcite{jeffrey_high_level_1996} use incidence matrices compiled from PrT-nets transitions for graphical visualisation of first-order Horn clause evaluation. In contrast, our work transforms Petri nets into logic programs and evaluates these programs by boolean matrix operations. Using logic programs allows the same implementation language to be shared between Petri nets, analysis and the logic reasoning system \cite{domenici_petri_1990}. 

\textbf{Bottom-up datalog evaluations.} Obtaining the least Herbrand model of recursive datalog programs can be reduced to computing the transitive closure of boolean matrices \cite{peirce_collected_1932,copilowish_matrix_1948}. Fischer and Meyer~\shortcite{fischer_boolean_1971} study a logarithmic divide-and-conquer computation technique by viewing relational databases as graphs and show a significant computational improvement over the ``naive'' iteration of matrix operators. A similar approach is explored by Ioannidis~\shortcite{ioannidis_computation_1986} for computing the fixpoint of recursive Horn clauses. Muggleton~\shortcite{muggleton_hypothesizing_2023} employs this approach in an ILP system called DeepLog. It repeatedly computes square binary matrices to quickly choose optimal background knowledge for deriving target hypotheses. While none of these methods is developed for Petri nets, we refer to relevant results in this paper. In our framework, we encode elementary nets as datalog programs and use boolean matrices to accelerate bottom-up datalog evaluations. 

\textbf{Relational algebra.} Recent work primarily studies the mapping of logic programs to linear equations over tensor spaces. Lin~\shortcite{lin_satisability_2013} gives a linear algebra framework for resolution and SAT problems. Grefenstette~\shortcite{grefenstette_towards_2013} uses a tensors-based calculus to represent truth values of domain entities, logical relations and operators for first-order logic. Sato~\shortcite{sato_linear_2017} shows that computing algebraic solutions can approximate recursive relations in first-order Datalog. Based on this approach, Sato et al.~~\shortcite{sato_abducing_2018} establish a linear algebra abduction framework by encoding a subset of recursive datalog programs in tensor space. However, limitations are identified for solving certain recursive programs due to the difficulty of computing some arithmetic solutions whereas these can be solved by iterative bottom-up evaluations \cite{sato_linear_2017}. Our approach does not have this issue since it uses boolean matrices for iterative bottom-up computation.
\section{Preliminaries}
\subsection{Elementary nets}
\begin{figure}[t]
    \centering
        \includegraphics[width=0.65\linewidth]{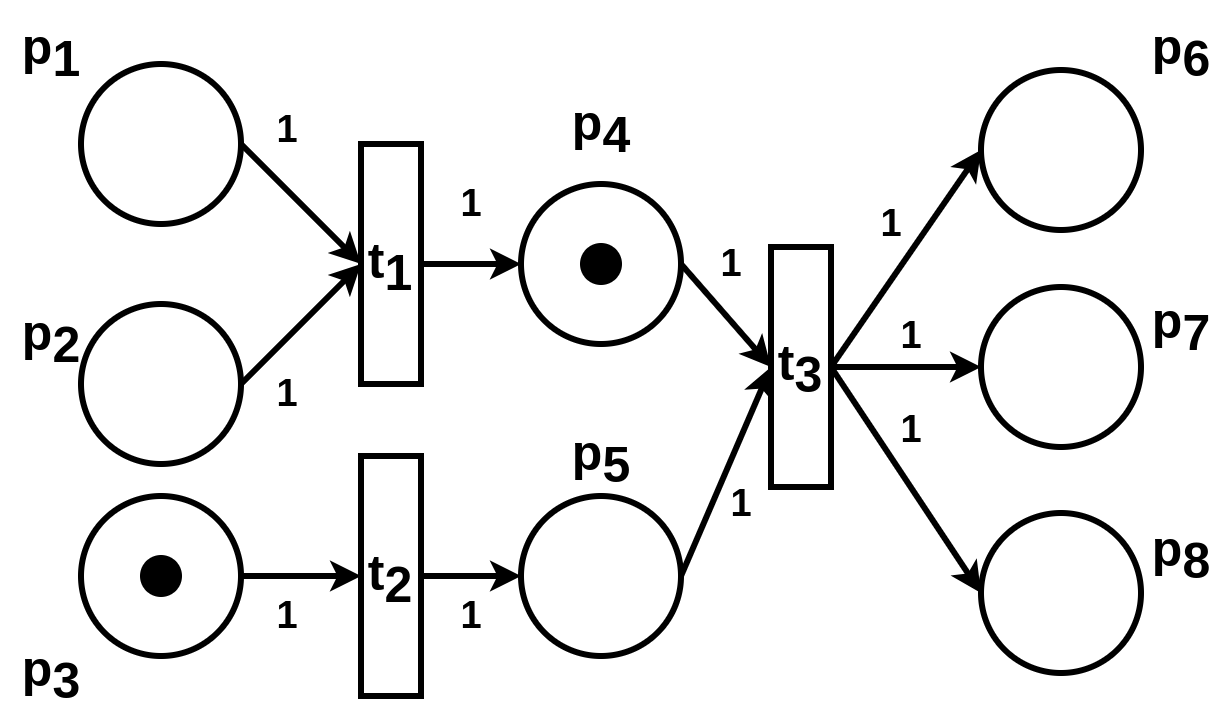} 
    \caption{An OEN. Places are ellipses. Transitions are rectangles. Arcs are marked with weights. A black dot is a token. The initial marking is $m(p_3) = m(p_4) = 1$ with a set representation $\{p_3, p_4\}$. Firing transition $t_2$ gives a new marking $\{p_4, p_5\}$.}
    \vspace{-10pt}
    \label{fig:petri_net}
\end{figure}
Elementary nets are a class of Petri nets \cite{rozenberg_elementary_1998}. A Petri net contains two sets of nodes: places $P$ and transitions $T$. The flow relation $F \subseteq (P \times T) \cup (T \times P)$ defines arcs between nodes. Places are marked by tokens. The weight of an arc $(p, t)$ or $(t, p)$ is conventionally written as $\overline{pt}$ or $\overline{tp}$. An elementary net is a tuple $(P, T, F)$ where $P \cap T = \emptyset$, $\overline{pt} = 1 \text{ if } (p,t) \in F$ and $\overline{tp} = 1 \text{ if } (t,p) \in F$. 

A one-bounded elementary net (OEN) has places which are marked with at most one token \cite{reisig_understanding_2013}. Figure \ref{fig:petri_net} shows an example of OEN with an initial marking of tokens. In this case, a marking (of tokens) $m$ can be represented by a boolean function $m: P \to \{0,1\}$ or a set alternatively. A marking is a binary value assignment in the sense that for a marking $m$, a place has a token if and only if $p \in m$ and $m(p) = 1$. A transition $t$ is enabled by $m$ if and only if for all arcs $(p,t)\in F$, $m(p) = 1$. A step $m \xrightarrow[]{t} m'$ is a fired transition resulting in a new marking according to: 
\begin{gather}
m'(p) = \left\{
\begin{aligned}
& m(p) - 1 \qquad\,\,\,\,\, \text{if } (p, t) \in F \\ 
& m(p) + 1 \qquad\,\,\,\,\, \text{if } (t, p) \in F \\
& m(p) \qquad\qquad\,\,\,\,\,\, \text{otherwise} \nonumber
\end{aligned}
\right.
\end{gather}
A marking $m'$ is reachable from a marking $m$ if and only a finite step sequence exists. A place is reachable from $m$ if and only if $m'$ is a reachable marking from $m$ and $m'(p) = 1$. A marking can be written as a vector ${[} 1, 0, 1, ... {]}^\intercal$ where each element is the binary value of the respective place. A transition can also be represented by a vector ${[} z_1, z_2, ..., z_k {]}^\intercal$ where $z_i$ changes the value of the output place. 



\subsection{Matrix algebra for datalog}

We first summarise the background on datalog programs \cite{LP_book}. A variable is a character string beginning with an uppercase letter. A predicate or constant symbol is a character string starting with a lowercase letter. A Horn clause is a disjunctive clause with at most one positive literal. A definite clause is a Horn clause with exactly one positive literal and a definite program is a set of definite clauses. A ground fact is a definite clause with no body literal and no variables. Datalog programs are definite programs without function symbols.  A first-order definite clause has the form, $A_0 \leftarrow A_1,...,A_n$. In this paper, we consider $A_i = r_i(x_1, ..., x_k)$ where $r_i$ is a predicate symbol and $x_j$ is a variable or constant symbol. 

The Herbrand base $\mathcal{B_P}$ is the set of all ground facts constructible from predicate and constant symbols in a first-order definite program $\mathcal{P}$. The Immediate Consequence Operator $\mathcal{T_P}$ \cite{van_emden_semantics_1976} is a mapping over subsets of $\mathcal{B_P}$ defined as:
\begin{flalign}
    & \forall I \subseteq \mathcal{B_P}, \mathcal{T_P} (I) = \{ \alpha \in \mathcal{B_P} \, | \, \alpha \gets B_1,..., B_m \, (m \geq 0) \nonumber \\
    & \text{ is a ground instance of a clause in } \mathcal{P} \text{ and } \{B_1,..., B_m\} \subseteq I\}. \nonumber  
\end{flalign}
$\mathcal{T_P}$'s least fixpoint is $\mathcal{T_P}{\uparrow^\omega}$ = $\cup_{n \geq 0}\mathcal{T_P}{\uparrow^n}$ where $\mathcal{T_P}{\uparrow^0} = \emptyset$ and $\mathcal{T_P}{\uparrow^n} = \mathcal{T_P}(\mathcal{T_P}{\uparrow^{n-1}})$. A clause entailed by $\mathcal{P}$ is written as $\mathcal{P} \models \alpha$. The minimal Herbrand model $M(\mathcal{P})$ contains all ground facts that are entailed by $\mathcal{P}$ and $M(\mathcal{P})$ = $\mathcal{T_P}{\uparrow^\omega}$. 

A recursive program has a body predicate that appears in the head of a clause. A recursive datalog program $\mathcal{P}$ can be compiled into matrices for bottom-up evaluation \cite{ceri_datalog_1989}. A linear and immediately recursive program $\mathcal{P}$ can be described by linear equations \cite{sato_linear_2017}:
\begin{flalign}
    &\textbf{R}_2^{0} = \textbf{0} \nonumber\\
    &\textbf{R}_2^{k} = \textbf{R}_1 + \textbf{R}_1\textbf{R}_2^{k-1}
    \label{eq:matrix_algebra}
\end{flalign}
where ground facts $r_1$ and a recursive clause $r_2$ can be represented as two $n \times n$ incidence matrices \textbf{R}$_1$, \textbf{R}$_2$ such that $(\textbf{R}_k)_{ij}$ = 1 if $\mathcal{P} \models r_k(c_i, c_j)$ for constant symbols $c_i, c_j$ and $(\textbf{R}_k)_{ij}$ = 0 otherwise. The transitive closure $\textbf{R}_2^*$ = $\textbf{R}_1 + \textbf{R}_1\textbf{R}_2^*$ has a solution $\textbf{R}_2^*$ = $\sum_{k=1}^{\infty}$\textbf{R}$_1^{k}$. Addition + between two boolean matrices is defined as (\textbf{A} + \textbf{B})$_{ij}$ = \textbf{A}$_{ij}$ $\lor$ \textbf{B}$_{ij}$.  Multiplication $\times$ is defined as $(\textbf{A} \times \textbf{B})_{ij}$ = $\bigvee_{k=1}^n$ \textbf{A}$_{ik}$ $\land$ \textbf{B}$_{kj}$ and can be abbreviated as (\textbf{A}\textbf{B}). When ``+'' and ``$\times$'' are boolean matrix operators, $\sum_{k=0}^{\infty}\textbf{A}^{k}$ can be efficiently computed by repeated matrix multiplications \cite{ioannidis_wong_1986}: 
\begin{flalign}
    \sum_{k=0}^{\infty}\textbf{A}^{k} = \prod_{k=0}^{\infty} (\textbf{I} + \textbf{A}^{2^k}) = (\textbf{I} + \textbf{A}) (\textbf{I} + \textbf{A}^2) (\textbf{I} + \textbf{A}^4) \, ... 
    \label{eq:repeated_squaring}
\end{flalign}
where $\textbf{I}$ is the identity. This logarithmic technique computes the transitive closure of $n \times n$ boolean matrices in $O(n^3 log_2 n)$ time complexity \cite{fischer_boolean_1971}.

\section{Transforming elementary nets}

We transform an OEN into linear and immediately recursive datalog. We map places into constant symbols and transitions into predicate symbols and add a recursive clause. A hypernode represents a set of place nodes and is created for every transition's input places. Hypernodes are denoted by new constant symbols that concatenate the member places' names. Given an OEN, the number of hypernodes we create is fixed and limited by the number of transitions.

\begin{definition} [\textbf{OEN transformation}]
    The transformation takes an OEN and an initial marking $m$. Let $t$ denote any transition. Input places of transitions are represented by hypernodes $p^*$. A datalog program $\mathcal{P}$ is written with predicate symbols $t$, $r_1$ and $r_2$ and variables X, Y and Z:
    \begin{itemize}
        \item write $t(p^*, q)$ if $t$ connects a hypernode $p^*$ and a place $q$
        \item write $t(p_1^*,p_2^*)$ if $t$ connects two hypernodes $p_1^*$ and $p_2^*$
        \item write $r_1(p,q) \gets t(p,q)$ for every $t$ clause written
        \item write $r_2(X,Y) \gets r_1(X, Y)$
        \item write $r_2(X,Y) \gets r_1(X, Z), r_2(Z, Y)$ 
        \item write $r1(m, p^*)$ if all members of hypernode $p^*$ are in the initial marking $m$  
         \nonumber
    \end{itemize}
    \label{def:transformation}
\end{definition}
\begin{example}
From the OEN in Figure \ref{fig:flight_routes}, transitions are represented by $flight\_1$, $flight\_2$ and $flight\_3$ clauses. We create two hypernodes $berlin\_paris$ and $london\_toronto$ for these transitions. These transitions connect places that are represented by constants $\{london$, $new\_york$, $toronto\}$. For an initial marking $m_1 = \{berlin, paris\}$, we can construct a datalog program:
\begin{gather}
\mathcal{P}_1:\left\{
\begin{aligned} 
    flight(m_1, berlin\_paris).&\\
    flight\_1(berlin\_paris, london\_toronto). &\\
    flight\_1(berlin\_paris, london). &\\
    flight\_1(berlin\_paris, toronto). &\\
    flight\_2(london\_toronto, new\_york). &\\
    flight\_3(new\_york, london). &\\
    flight(X,Y) \gets flight\_1(X,Y). &\\
    flight(X,Y) \gets flight\_2(X,Y). &\\
    flight(X,Y) \gets flight\_3(X,Y). &\\
    route(X,Y) \gets flight(X,Y). &\\ 
    route(X,Y) \gets flight(X,Z), route(Z,Y). & \nonumber
\end{aligned}
\right\}
\end{gather} 

\label{ex:transformation}
\end{example}
For the marking $m_2 = \{berlin\}$ where only Berlin has a token, no flight clause is added because $berlin$ is not a hypernode. We can use $\mathcal{P}_1$ to find reachable places (cities). 
\begin{example} (Continued).
    In Figure \ref{fig:flight_routes}, $new\_york$ and $toronto$ are reachable from the marking $m_1 = \{ berlin, paris \}$ via transitions $flight_1$ and $flight_2$ but is not reachable from the marking $m_2 = \{berlin\}$. We can find groundings of $route(m_1, Y)$ and $route(m_2, Y)$ in the Herbrand model $M(\mathcal{P}_1)$ which gives us $route(m_1, new\_york)$ and $route(m_1, toronto)$.
\end{example}

\begin{theorem}
    Let $\mathcal{P}$ be a datalog program transformed from an OEN = (P, T, F). For constant symbols $p_1$, $p_2$ in $\mathcal{P}$, $\mathcal{P} \models r(p_1,p_2)$ if and only if $p_2$ is reachable from $p_1$ in the OEN.
    \label{theorem:equivalence}
\end{theorem}
\begin{proof}
    We provide a sketch proof. Since a marking maps places to binary values, we can consider a marking as a truth value assignment, e.g. $m(p) = 1$ means $m(p) = \top$ where $m$ is a marking and $p$ is a place. When a transition $t$ fires, a step leads to a marking that assigns new truth values to places. This allows us to define an operator $\mathcal{T}_F$ over the power set $\Omega = 2^P$ given the OEN = (P, T, F): 
    \begin{flalign}
        \forall m \in \Omega, \mathcal{T}_F(m) = \{ p_j \in P \, | \, m'(p_j) \gets \land_{(p_i, t) \in F} m(p_i) &\nonumber \\ 
        \text{ where} \, \{p_1, ..., p_n \} \subseteq m \text{ for } 1 \leq i \leq n, (t, p_j) \in F, m' \in \Omega \}& \nonumber
    \end{flalign}
    $(\mathcal{T}_F, \subseteq)$ is a complete lattice since $\Omega$ is a set partially ordered by the set inclusion $\subseteq$ where $\mathcal{T}_F{\uparrow^0} = \emptyset$ and $\mathcal{T}_F{\uparrow^n} = \mathcal{T}_F(\mathcal{T}_F{\uparrow^{n-1}})$. The top element is the set of all places $P$ and the bottom element is $\emptyset$.
    We can prove $\mathcal{T}_F$ is a monotonic and compact operator which allows us to show that $\mathcal{T}_F$ has a least fixpoint equal to $\mathcal{T}_F{\uparrow^\omega} = \cup_{n \geq 0}\mathcal{T}_F{\uparrow^n}$. We use this to prove the two implications of the theorem. If $p_2$ is reachable from $p_1$ in the OEN, two markings $m_0$ and $m_k$ exist in $\mathcal{T}_F{\uparrow^\omega}$ such that $m_0(p_1) = \top = m_k(p_2)$. $m_k$ is reachable from $m_0$ via a sequence of steps (fired transitions) which corresponds to ground facts in the transformed program. Therefore, we know that $r(p_1,p_2) \in \mathcal{T_P}{\uparrow^\omega}$ so $\mathcal{P} \models r(p_1,p_2)$. The other implication can be proven similarly.
\end{proof}

We defer the full proof of Theorem \ref{theorem:equivalence} to Appendix. Theorem \ref{theorem:equivalence} demonstrates that if we can derive a ground fact entailed by a transformed program then we know the relevant places in the OEN are reachable. Theorem \ref{theorem:equivalence} allows us to compute reachability in an OEN by evaluating the transformed program via boolean matrices in BMLP methods. 

\section{Boolean matrix logic programming}
\subsection{Problem setting}
In contrast to traditional logical inferences, a Boolean Matrix Logic Programming problem involves ``programming'' low-level boolean matrices to evaluate datalog programs.  
\begin{definition} [Boolean Matrix Logic Programming (BMLP) problem]
    Let $\mathcal{P}$ be a datalog program containing a set of clauses with predicate symbol $r$. The goal of Boolean Matrix Logic Programming (BMLP) is to find a boolean matrix $\textbf{R}$ encoded by a datalog program such that $(\textbf{R})_{ij}$ = 1 if $\mathcal{P} \models r(c_i, c_j)$ for constant symbols $c_i, c_j$ and $(\textbf{R})_{ij}$ = 0 otherwise.
\end{definition}
In this work, we focus on a subset of this problem, arity two linear and immediately recursive datalog programs with at most two body literals. We present a BMLP framework to compute the transitive closure of boolean matrices for evaluating this class of programs. 

\subsection{Framework}

\subsubsection{Boolean matrix compilation}
To compile a datalog program $\mathcal{P}$, we first create mappings between constant symbols and integer numbers, $cton$ (constant symbols to numbers) and $ntoc$ (numbers to constant symbols). The constant symbols are ordered from 0 and the i-th symbol $c_i$ is represented by $cton(i, c_i)$ and $ntoc(c_i, i)$. We compile a set of ground facts $\{v(0, b_0), v(1, b_1), ...\}$ ordered from 0 to describe a clause $r$ as a boolean matrix. Each $v(i, b_i)$ is the i-th row of the boolean matrix and $b_i$ is a binary code written as an integer such that the j-th bit of the binary code $(b_i)_j$ is 1 if $\mathcal{P} \models r(c_i, c_j)$ and 0 otherwise.
\begin{example}
We can compile a boolean matrix $\textbf{R}$ from $\mathcal{P}_3$:
\begin{gather}
\mathcal{P}_{3}:\left\{
\begin{aligned}
    flight(c_0, c_1). \,\, flight(c_1, c_2). &\\
    route(X,Y) \gets flight(X,Y). &\\ 
    route(X,Y) \gets flight(X,Z), route(Z,Y). &\nonumber
\end{aligned}
\right\}
\label{ex:recursive_program}
\end{gather}
\vspace{-5pt}
\begin{flalign*}
    \mathcal{P}_3 \models flight(c_0, c_1) \Longrightarrow \text{matrix row 0, } b_0 = 010 \Longrightarrow v(0, 2)&\\
    \mathcal{P}_3 \models flight(c_1, c_2) \Longrightarrow \text{matrix row 1, } b_1 = 100 \Longrightarrow v(1, 4)&\\
    \Longrightarrow \text{matrix row 2, } b_2 = 000 \Longrightarrow v(2, 0)&
\end{flalign*}
\vspace{-10pt}
\begin{gather}
\textbf{R}: \left\{
\begin{aligned}
    cton(0, c_0). \,\, ntoc(c_0, 0). \,\, cton(1, c_1). \,\,  ntoc(c_1, 1). \\
    cton(2, c_2). \,\, ntoc(c_2, 2). \,\, v(0, 2). \,\, v(1, 4). \,\, v(2, 0). \nonumber
\end{aligned}
\right\}
\end{gather}
\end{example}

We compile vectors as single-row matrices. When $c_i, c_j$ ($0 \leq i, j < n$) come from the same set of constant symbols, a boolean matrix created this way is a $n \times n$ square matrix. 

\subsubsection{Iterative extension (BMLP-IE)}

\begin{algorithm}[t]
    \caption{Iterative extension (BMLP-IE)}
    \label{alg:algorithm1}
    \textbf{Input}: A $1 \times n$ vector $\textbf{v}$, two $k \times n$ boolean matrices $\textbf{R}_{1}$, $\textbf{R}_{2}$ that encode ground facts.\\
    \textbf{Output}: Transitive closure $\textbf{v}^*.$
    \begin{algorithmic}[1] 
        \STATE Let $\textbf{v}^*$ = $\textbf{v}$, $\textbf{v}'$ = \textbf{0}.
        \WHILE{$True$}
            \FOR{$1 \leq i \leq k$} 
                \STATE the i-th bit of $\textbf{v}'$ is 1 if $\textbf{v}$ AND $(\textbf{R}_{1})_{i,*}$ == $(\textbf{R}_{1})_{i,*}$. 
            \ENDFOR
        \STATE $\textbf{v}^*$ = ($\textbf{v}'$ $\textbf{R}_{2}$)+ \textbf{v}.
        \STATE $\textbf{v} = \textbf{v}^*$ if $\textbf{v}^* \ne \textbf{v}$ else break.
        \ENDWHILE
    \end{algorithmic}
\end{algorithm}

The iterative extension algorithm (BMLP-IE) iteratively expands the set of derivable ground facts for a partially grounded clause $r(c, Y)$. Given a program with $n$ constant symbols and $k$ ground facts, we represent the collection of ground facts as a $1 \times n$ vector or a $n$-bit binary code $\textbf{v}$ such that $(\textbf{v})_i = 1$ if $r(c, c_i)$ holds for the i-th constant symbol $c_i$. Then, we number ground facts in the program and write two $k \times n$ matrices $\textbf{R}_1$ and $\textbf{R}_2$. $\textbf{R}_1$ and $\textbf{R}_2$ describe which constants appear first and second in the ground facts, e.g. $first(i, c_j)$ says that j-th constant symbol appears as the first argument in the i-th ground fact.

We implement Algorithm \ref{alg:algorithm1} to compute the transitive closure of $\textbf{v}$. We use the bitwise AND to locate rows in $\textbf{R}_1$ that contain all the 1 elements in $\textbf{v}$. This finds indices of other ground facts that share constant symbols with ground facts represented by $\textbf{v}$. These derivable ground facts are added to $\textbf{v}$ via boolean matrix multiplication and addition. We iterate this process until we find the transitive closure of $\textbf{v}$.

\begin{proposition}
    Given a $1 \times n$ vector $\textbf{v}$, two $k \times n$ boolean matrices $\textbf{R}_{1}$ and $\textbf{R}_{2}$, Algorithm \ref{alg:algorithm1} has a time complexity $O(k n^2)$ for computing the transitive closure $\textbf{v}^*$.
    \label{proposition:1}
\end{proposition}
\begin{proof}
    The ``while'' loop in Algorithm \ref{alg:algorithm1} runs with $O(k \times n)$ bitwise operations due to multiplications between a vector and a $k \times n$ matrix. Until we find the transitive closure, at least one ground fact needs to be added at each iteration. Therefore, there are at most $n$ iterations which require $O(k \times n^2)$ bitwise operations. 
\end{proof}

\subsubsection{Repeated matrix squaring (BMLP-RMS)}
A linear and immediately recursive program, such as program $\mathcal{P}_3$ in Example \ref{ex:recursive_program}, can be represented by Equation (\ref{eq:matrix_algebra}). For $n$ constant symbols and ground facts with a predicate symbol $r$, we compile a $n \times n$ matrix $\textbf{R}_1$ such that $(\textbf{R}_1)_{ij}$ = 1 if $\mathcal{P} \models r(p_i, p_j)$ and otherwise $(\textbf{R}_1)_{ij}$ = 0. We can find the least Herbrand model for the recursive program by computing $\sum_{k=1}^{\infty}$\textbf{R}$_1^{k}$ based on Equation (\ref{eq:repeated_squaring}), which has an alternative form \cite{ioannidis_computation_1986}:
\begin{flalign}
    \sum_{k=0}^{\infty}\textbf{R}_1^{k} 
    = \lim_{k\to\infty} (\textbf{I} + \textbf{R}_1)^{k} 
    = (\textbf{I} + \textbf{R}_1) (\textbf{I} + \textbf{R}_1) ...
    \label{eq:simple_recursion}
\end{flalign}

Equation (\ref{eq:simple_recursion}) performs boolean matrix operations on the same elements in the matrix $(\textbf{I} + \textbf{R}_1)$. To avoid this, we take a similar approach to the logarithmic technique in Equation (\ref{eq:repeated_squaring}) to skip computations by repeatedly squaring matrix products.
\begin{flalign}
    \sum_{k=0}^{\infty}\textbf{R}_1^{k}
    &=(\textbf{I} + \textbf{R}_1) (\textbf{I} + \textbf{R}_1) (\textbf{I} + \textbf{R}_1)^2 (\textbf{I} + \textbf{R}_1)^4 ... \nonumber\\
    &=\prod_{i=1}^{\infty}(\textbf{I} + \textbf{R}_1)^{2^k}
    \label{eq:repeated_squaring2}
\end{flalign}
We re-arrange Equation (\ref{eq:simple_recursion}) as Equation (\ref{eq:repeated_squaring2}). At iteration $i$, we can compute $(\textbf{I} + \textbf{R}_1)^{2^i}$ by squaring the matrix product from the previous iteration until we reach the transitive closure. We implemented this process in Algorithm \ref{alg:algorithm2}. 

\begin{proposition} 
    Given a $n \times n$ boolean matrix $\textbf{R}_{1}$, Algorithm \ref{alg:algorithm2} has a time complexity $O(n^3 log_2 n)$ for computing the transitive closure $\textbf{R}_{2}^*$.
    \label{proposition:2}
\end{proposition}
\begin{proof}
    The time complexity of native multiplications between two $n \times n$ matrices is $O(n^3)$. Finding the closure requires $O(n)$ boolean matrix multiplications \cite{fischer_boolean_1971}. Since Algorithm \ref{alg:algorithm2} takes steps to the power of two, this number is reduced to $O(log_2 n)$, giving an overall $O(n^3 log_2 n)$ bitwise operations.
\end{proof}

The time complexity of matrix multiplication could be theoretically improved to $O(n^{2.376})$ \cite{coppersmith_matrix_1990}. Since we operate at the bit level, optimisations may not have a significant return in actual runtime. This paper does not investigate this tradeoff.

\begin{algorithm}[t]
    \caption{Repeated matrix squaring (BMLP-RMS)}
    \label{alg:algorithm2}
    \textbf{Input}: A $n \times n$ boolean matrix $\textbf{R}_1$ that encodes ground facts.\\
    \textbf{Output}: Transitive closure $\textbf{R}_2^*.$
    \begin{algorithmic}[1]
        \STATE Let $\textbf{R} = \textbf{R}_2^* = \textbf{I} + \textbf{R}_1$.
        \WHILE{$True$}
        \STATE $\textbf{R}_2^* = \textbf{R}^2$.
        \STATE $\textbf{R} = \textbf{R}_2^*$ if $\textbf{R}_2^* \ne \textbf{R}$ else break.
        \ENDWHILE
    \end{algorithmic}
\end{algorithm}


\subsection{Computing reachability}
BMLP-IE and BMLP-RMS are reminiscent of incidence matrix methods for computing graph reachability in \cite{fischer_boolean_1971,ioannidis_computation_1986}. BMLP-IE traverses a graph from a starting node and BMLP-RMS finds the connectivity between all nodes. Therefore, we devise these algorithms for two different reachability tasks in OENs. We implement BMLP-IE and BMLP-RMS in SWI-Prolog (version 9.2).\smallskip

\noindent\textbf{Reachable places from a specific marking.} We use BMLP-IE (Algorithm \ref{alg:algorithm1}) to return the transitive closure of reachable places from an initial marking (a partially grounded clause). We compile a vector $\textbf{v}$ to represent the initial marking (a set of ground facts) and two incidence matrices to describe the input and output places of transitions (constants that appear as the first and second argument of ground facts). The goal is to compute the transitive closure $\textbf{v}^*$.\smallskip

\noindent\textbf{Reachable places from all markings.} BMLP-RMS (Algorithm \ref{alg:algorithm2}) finds the union of all reachable places (the least model of the transformed program). We compile an incidence matrix $\textbf{R}_1$ by considering transitions as edges. BMLP-RMS returns a square boolean matrix $\textbf{R}_2^*$ representing places reachable from every initial marking.
\section{Experiments}
This section presents experimental results from OENs in two domains, graph data and a biology task. For Prolog systems, tabled evaluations ensure termination and faster grounding re-evaluations. Since our BMLP methods are implemented in SWI-Prolog, we compare against SWI-Prolog (version 9.2) with tabling. We also compare with state-of-the-art Prolog systems with tabling implementations, B-Prolog (version 8.1) and XSB-Prolog (version 5.0). We also test against Clingo (version 5.6), a state-of-the-art ASP system for finding stable models \cite{gelfond_stable_1988}, as an alternative approach to Prolog. By transforming OENs into datalog programs, we examine if applying BMLP on these datalog programs is faster than evaluating these programs by Clingo, B-Prolog, SWI-Prolog and XSB-Prolog with tabling. We compute each method's mean CPU time (excluding compilation time) in seconds and report standard deviations in Appendix.  

Our experiments\footnote{All experiments have been performed on a single thread with Intel(R) Core(TM) i7-8665U CPU @ 1.90GHz on a laptop. This underlines the efficiency of the BMLP methods.} answer the following questions:\smallskip

\noindent \textbf{Q1}: Can BMLP reduce the runtime in computing reachability in one-bounded elementary nets?\smallskip

BMLP-IE and BMLP-RMS are devised for two different tasks. In Section \ref{sec:exp_1}, we focus on evaluating BMLP-IE's runtime performance in identifying reachable places from a specific marking against ASP and Prolog systems. Then, BMLP-RMS is compared with ASP and Prolog systems for computing reachable places from all markings.\smallskip

\noindent \textbf{Q2}: How well do BMLP methods scale with large Petri nets?\smallskip

Let $k$ denote the number of ground facts and $n$ be the number of constant symbols in a datalog program. Proposition \ref{proposition:1} and \ref{proposition:2} identify two factors affecting the time complexity of our BMLP methods: the number of transitions ($k$) and places ($n$). BMLP-IE's time complexity is affected by $k$ and $n$. The time complexity of BMLP-RMS is just influenced by $n$. We examine how these factors empirically impact BMLP methods' runtime. \smallskip

\noindent \textbf{Q3}: Is BMLP applicable to a real-world problem? \smallskip

Petri nets play a vital role in understanding the interactions between the components of a biological system \cite{sahu_advances_2021}. In Section \ref{sec:exp_2}, we examine a type of biological model called genome-scale metabolic network models. Genome-scale metabolic models encapsulate known biochemical reaction information within a microorganism. BMLP-IE is more suitable for identifying the interactions between reactions given different combinations of compounds. We evaluate its runtime performance against ASP and Prolog systems for an OEN created from a genome-scale metabolic network.\smallskip

\subsection{Airplane flight routes}
\label{sec:exp_1}

\noindent \textbf{Computing OEN reachability:} Reachability in airplane flight route database is a typical transitive closure relationship. We create artificial OENs describing randomly generated airplane flight routes and represent them as linear and immediately recursive programs. Transitions in OENs are randomly sampled from pairs of places with a probability $p_t$. We then transform the OENs into similar programs like $\mathcal{P}_2$ in Example \ref{ex:transformation}.
We map the place and transition names into constant symbols. Ground facts are generated from pairs of constant symbols and are compiled into boolean matrices.\smallskip

\noindent \textbf{Methods:} The runtime of BMLP-IE and BMLP-RMS will be examined in \textbf{sub-tasks 1} and \textbf{2}. Experiments are repeated 10 times. Runtimes are capped at 300 seconds. 
\begin{enumerate}
    \item \textbf{Reachable places from one city.} Given a city $c$, we use B-Prolog, SWI-Prolog, XSB-Prolog and Clingo to enumerate groundings of $route(c, Y)$. 
    BMLP-IE compiles ground facts into a vector and computes the transitive closure of the vector. We randomly sample OEN transition with probability $p_t \in$ \{1, 0.5, 0.1, 0.01, 0.001, 0.0001\}. This results in programs with different numbers of ground facts. 
    \item \textbf{Reachable places from all cities.} We use B-Prolog, SWI-Prolog, XSB-Prolog and Clingo to find the least model of the datalog program. BMLP-RMS compile ground facts into a square boolean matrix and computes its transitive closure. We vary the number of places $n \in$ \{1000, 2000, 3000, 4000, 5000\}. This changes the size of matrices, up to 5000 $\times$ 5000. 
\end{enumerate}
 \smallskip

\begin{figure}[t]
    \centering
        \includegraphics[width=0.8\linewidth]{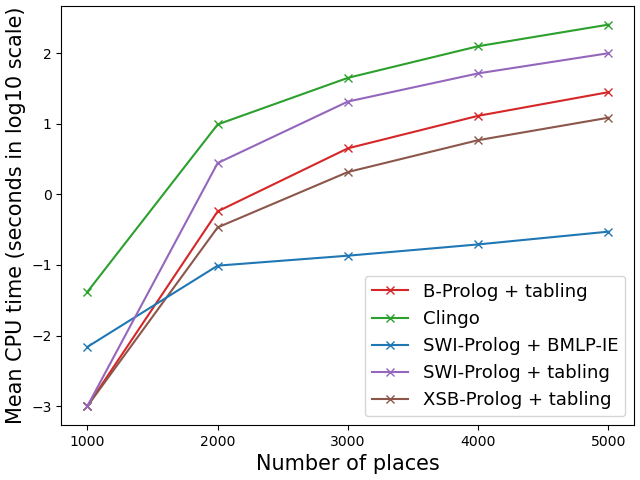} 
        \vspace{-10pt}
    \caption{Runtime when the probability $p_t$ of transitions (ground facts) is 0.001 but the No. places (constant symbols) $n$ varies. }
    \label{fig:exp_1_runtime_2}
    \vspace{-5pt}
\end{figure}

\begin{figure}[t]
    \centering
        \includegraphics[width=0.8\linewidth]{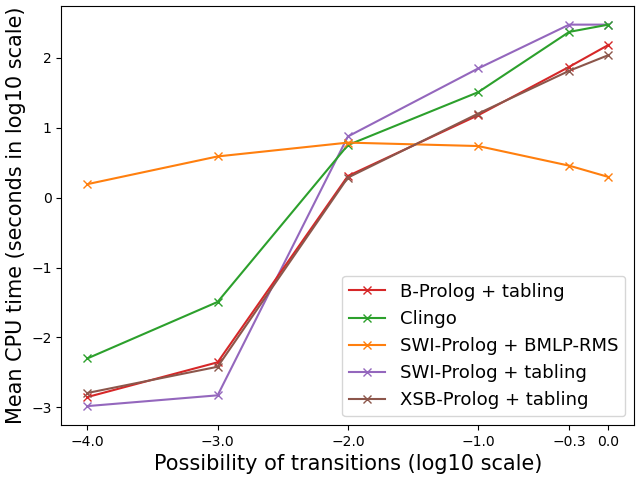} 
        \vspace{-10pt}
    \caption{Runtime when No. places (constant symbols) $n$ = 1000 and the probability $p_t$ of transitions (ground facts) varies. }
    \label{fig:exp_1_runtime}
    \vspace{-5pt}
\end{figure}

\noindent \textbf{Results:} Runtimes are shown in base-10 logarithms for better contrasts. Figure \ref{fig:exp_1_runtime_2} shows runtimes for \textbf{sub-task 1}. BMLP-IE (blue curve) is 40 times faster than XSB-Prolog and 800 times faster than Clingo when the number of places (constant symbols) $n = 5000$. For small probability $p_t = 0.001$, the number of transitions (ground facts) is close to the number of places, bounded by $O(n)$. When the probability $p_t$ is small, Proposition \ref{proposition:1} says that BMLP-IE performs close to $O(n^3)$ bitwise operations. This is shown by BMLP-IE's slower runtime growth over an increasing number of places. 

Figure \ref{fig:exp_1_runtime} shows runtimes for \textbf{sub-task 2}. BMLP-RMS's runtimes (orange curve) have a clear upper bound. BMLP-RMS runtime does not increase significantly, which supports Proposition \ref{proposition:2} since its time complexity is not affected by the number of transitions (ground facts). While Clingo and SWI-Prolog time out when $p_t = 1$, BMLP-RMS is 50 times faster than the second fastest XSB-Prolog. Notably, BMLP-RMS's runtime even reduces when $p_t > 0.01$. An explanation is that in denser elementary nets, the shortest paths between places have lengths closer to 1, therefore fewer squaring operations are needed to reach transitive closure. 

Writing boolean matrix outputs as Prolog programs creates overhead. This overhead is more significant for smaller problems, reflected by BMLP-IE and BMLP-RMS's slower runtime than the other methods when the number of places is 1000 in Figure \ref{fig:exp_1_runtime_2} and \ref{fig:exp_1_runtime}. While BMLP-IE and BMLP-RMS are used for different tasks, results also show their complementary nature. BMLP-IE has a faster runtime when the OEN is sparse (the probability of transitions or ground facts is low), and BMLP-RMS is suitable for OENs that are highly connected (the probability of transitions or ground facts is high). 

\subsection{Metabolic network models in biology}
\label{sec:exp_2}

\noindent \textbf{Computing producible substrates:} Petri nets are commonly used formalism for representing complex biological systems such as genome-scale metabolic network models and diagnosing their dynamic behaviours \cite{sahu_advances_2021}. A genome-scale metabolic network model contains biochemical reactions of chemical substrates $x_i$ and $y_j$: 
\begin{flalign}
    \text{Irreversible: } x_1,x_2,...,x_m &\longrightarrow y_1,y_2,...,y_n \nonumber\\
    \text{Reversible: } x_1,x_2,...,x_m &\longleftrightarrow y_1,y_2,...,y_n \nonumber
\end{flalign}
Reversible reactions are treated as two reactions in opposite directions. A genome-scale metabolic network is an OEN in the sense that a reaction is a transition, and reactant and product substrates are places. The behaviours of metabolic network models vary according to accessible nutrients. Therefore, we aim to find producible substrates from the reactions by defining nutrients as initial markings. We use the most comprehensive genome-scale metabolic network model to date \cite{iML1515}, iML1515, which comprises 2719 metabolic reactions and 1877 metabolites. \smallskip 

\noindent \textbf{Method:} To transform the metabolic network model into datalog, we create ground facts and constant symbols by concatenating substrates on the left-hand and right-hand sides of reactions. A marking is introduced as a set of pseudo-reactions. For example, a reaction and a marking become:
\begin{flalign}
    2 H_2 + O_2 \longrightarrow 2 H_2 O &\Longrightarrow reaction(h2\_o2, h2o).\nonumber\\
    m_1: \{ H_2, O_2, H_2 O\} &\Longrightarrow reaction(m_0,h2\_o2\_h2o). \nonumber\\
    &\Longrightarrow reaction(h2\_o2\_h2o, h2\_o2).\nonumber
\end{flalign}
Reachable markings are sets of substrates that can be produced via metabolic reactions. The transitive closure of metabolic paths from an initial marking of substrates can be represented by the following program. 
\begin{gather}
\mathcal{P}_3:\left\{
\begin{aligned}
    metabolic\_path(X,Y) \gets reaction(X,Y).& \nonumber\\ 
    metabolic\_path(X,Y) \gets reaction(X,Z),&  \nonumber\\
    metabolic\_path(Z,Y).& \nonumber
\end{aligned}
\right\}
\end{gather}
We test different initial markings by randomly sampling 1, 10, 100, and 1000 nutrients from all substrates. The number of transitions (ground facts) is around 10000 and the number of places (constant symbols) is 4500, so the probability of transitions is very low ($p_t = \frac{10000}{4500^2} = 0.0005$). We compare BMLP-IE with B-Prolog, SWI-Prolog, XSB-Prolog and Clingo. We compute the mean CPU time (without compilation time) across 100 repeats.\smallskip

\noindent \textbf{Results: } We report runtimes in base-10 logarithms for clearer comparisons. In Figure \ref{fig:exp_1_runtime_4}, we observe that BMLP-IE (blue curve) has a higher mean runtime compared with tabled XSB-Prolog, SWI-Prolog and B-Prolog when the number of substrates is less than 100. This is due to the boolean matrix writing overhead. However, when the number of substrates exceeds 100, BMLP-IE is 10 times faster than the second fastest XSB-Prolog and 600 times faster than Clingo. When the initial marking contains more substrates, more reactions can fire. This will significantly increase the problem complexity for B-Prolog, SWI-Prolog and XSB-Prolog since they need to recursively evaluate longer metabolic paths. Clingo also needs to search for large stable models which increases runtime. BMLP-IE's slow increase in runtime aligns with Proposition \ref{proposition:1} again as low $p_t$ can reduce its time complexity close to $O(n^3)$ bitwise operations. Since BMLP-IE can expand multiple metabolic paths simultaneously based on boolean matrices, the low saturation of transition means that it can find the transitive closure in a few iterations. 

\begin{figure}[t]
    \centering
        \includegraphics[width=0.8\linewidth]{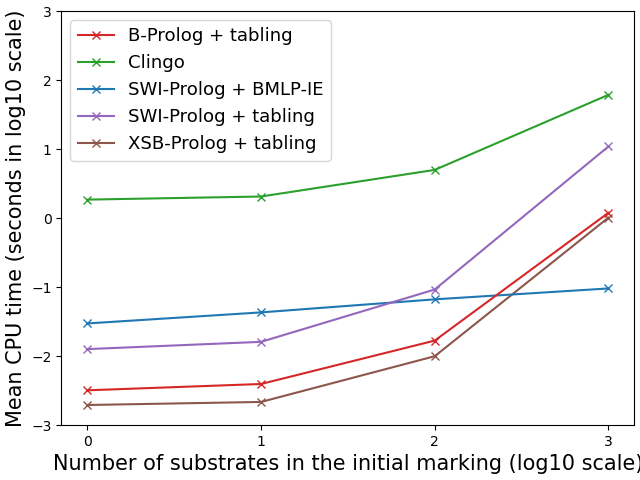} 
    \caption{Runtime for computing producible substrates from reactions in the genome-scale metabolic network model iML1515. }
    \label{fig:exp_1_runtime_4}
    \vspace{-5pt}
\end{figure}
\section{Conclusion and future work}
Petri nets complement knowledge representation with dynamic analysis. However, the simulation of Petri nets in logic programming has been much less explored. A significant challenge is scaling to large Petri nets. We address this by introducing Boolean Matrix Logic Programming (BMLP), which utilises low-level boolean matrices encoded in Prolog to compute datalog programs. In our BMLP framework, we implement two algorithms to find the transitive closure of boolean matrices for evaluating recursive programs. This framework is applied to compute reachability in OENs from two domains, knowledge graphs and a metabolic network model. Our theoretical and empirical results show that BMLP methods are especially suited for extensive OENs. BMLP-IE and BMLP-RMS can provide a 40 times faster runtime than state-of-the-art ASP solver and Prolog systems with tabling. Additionally, BMLP can be augmented with parallel processing to enhance the performance of matrix operations.\smallskip

\noindent \textbf{Limitations and future work:} In this work, only a subset of recursive programs is considered. We aim to extend the class of programs we can evaluate with BMLP to leverage the extensive literature on models such as Petri nets. In addition, we assume that datasets are noise-free. A natural link exists between Petri nets and Probabilistic Logic Programming \cite{de_raedt_probabilistic_2015}. Transitions sometimes have firing likelihood constraints and the uncertainty can be modelled by probabilistic logic programs. Evaluating these programs requires a tensor-based algebra framework.

\newpage
\section*{Ethical Statement}
Datasets used in empirical studies do not contain any personal information or other sensitive content. There are no ethical issues.

\section*{Acknowledgments}
The first, third and fourth authors acknowledge support from the UKRI 21EBTA: EB-AI Consortium for Bioengineered Cells \& Systems (AI-4-EB) award (BB/W013770/1). The second author acknowledges support from the UK’s EPSRC Human-Like Computing Network (EP/R022291/1), for which he acts as Principal Investigator.

\bibliographystyle{named}
\bibliography{bmlp}

\appendix
\renewcommand{\thesection}{\Alph{section}}
\section*{Appendix}
The appendix includes the following sections:
\begin{itemize}
    \item Section \ref{sec:appendix_proofs} provides the proof for \textbf{Theorem 1} which demonstrates the correctness of OEN transformation into datalog for reachability analysis. 
    \item Section \ref{sec:appendix_matrices} show examples of boolean matrix compilation in experiments from Section 6 (main paper).
    \item Section \ref{sec:appendix_exp} contains examples of experimental material from Section 6 (main paper). We aim to open-source all of our experimental data and implementations for reproducibility. 
    \item Section \ref{sec:appendix_results} present the mean and standard deviation of runtimes in experiments on iterative extension (BMLP-IE), repeated matrix squaring (BMLP-RMS), B-Prolog \cite{zhou_language_2012}, SWI-Prolog \cite{wielemaker:2011:tplp}, XSB-Prolog \cite{swift_xsb_2012} and Clingo \cite{gebser_clingo_2014} from Section 6 (main paper). 
\end{itemize}
\label{sec:appendix_intro}
\section{Correctness of OEN transformation}
\label{sec:appendix_proofs}
In this section, we prove Theorem \ref{theorem:appendix_equivalence}, which demonstrates the correctness of computing reachable places in a one-bounded elementary net (OEN) by evaluating a linear and immediately recursive program transformed from the OEN. In the main paper, we present a \textit{Boolean Matrix Logic Programming} (BMLP) framework for evaluating linear and immediately recursive programs. Theorem \ref{theorem:appendix_equivalence} provides a theoretical basis for using our BMLP framework.

We refer readers to Definition 1 in the main paper. The transformation in Definition 1 maps an OEN into a linear and immediately recursive program by writing places as constant symbols and transitions as predicate symbols. A recursive clause is added to the program. The transformed program is linear and immediately recursive since it contains only one recursive clause and the recursive predicate appears only once in the head of this clause. 

To prove Theorem \ref{theorem:appendix_equivalence}, we first show the effect of a step (fired transition) on updating markings. Since a marking is a binary function, we can map its range to $\{\bot,\top\}$.  A step in an OEN changes the current marking and updates the truth values assigned to places. 

\begin{lemma}
    Let $OEN = (P,T,F)$. For a transition $t \in T$, marking $m$ and $m'$, $\land_{(p_i, t)\in F} m(p_i) \to \land_{(t, p_j) \in F} m'(p_j)$.
    \label{lemma:step}
\end{lemma}
\begin{proof}
    For an arbitrary $t \in T$, assume $\land_{(p_i, t)\in F} m(p_i)$. $m(p_i) = 1$ and $p_i$ has a token, so $t$ is enabled. Since places in OEN can have at most one token, for all $(t, p_j) \in F$, $m'(p_j) = m(p_j) \, + \, 1$ is maximally 1. This is equivalent to the truth value $\top$, which means that $p_j$ has a token. Therefore, every place connected to $t$ has a token $\land_{(t, p_j) \in F} m'(p_j)$.
\end{proof}
Lemma \ref{lemma:step} shows that when a transition $t$ fires, a step updates the truth value of places. As a consequence of a step (a fired transition), in the new marking, tokens are re-assigned to other places whose truth values now become $\top$. Now, we consider the set of all token-to-place assignments. This set is a power set $\Omega = 2^P$ and it is the union of all possible markings (each marking is also a set of places). From the property shown in Lemma \ref{lemma:step}, we define an operator called step consequence $\mathcal{T}_F: \Omega \to \Omega$.

\smallskip\noindent\textbf{Definition 3 (Step consequence operator $\mathcal{T}_F$)} \textit{Given an OEN = (P, T, F), $\mathcal{T}_F$ is an operator over the power set $\Omega$: }
    \begin{flalign}
        \forall m \in \Omega, \mathcal{T}_F(m) = \{ p_j \in P \, | \, m'(p_j) \gets \land_{(p_i, t) \in F} m(p_i) &\nonumber \\ 
        \text{ where} \, \{p_1, ..., p_n \} \subseteq m \text{ for } 1 \leq i \leq n, (t, p_j) \in F, m' \in \Omega \}& \nonumber
    \end{flalign}

Definition 3 says that the result of applying $\mathcal{T}_F$ to a marking is a set of places reachable from it. $\mathcal{T}_F(\emptyset)$ only includes the initial marking due to the antecedent of the implication being false. To obtain the union of all places reachable from an empty marking, we can apply $\mathcal{T}_F$ repeatedly to its results. The recursive application of $\mathcal{T}_F$ can be defined as $\mathcal{T}_F{\uparrow^0} = \emptyset$ and $\mathcal{T}_F{\uparrow^n} = \mathcal{T}_F(\mathcal{T}_F{\uparrow^{n-1}})$. This corresponds to a lattice.

\smallskip\noindent\textbf{Proposition 3} $(\mathcal{T}_F, \subseteq)$ is a complete lattice.
\begin{proof}
    The proof follows from Definition 3. $(\mathcal{T}_F, \subseteq)$ is a complete lattice since $\Omega$ is a set partially ordered by set inclusion $\subseteq$ where $\mathcal{T}_F$ is an operator over the power set $\Omega$. The top element $P$ is the union of elements in $\Omega$ and the bottom element $\emptyset$ is the intersection of elements in $\Omega$. 
\end{proof}

Let $(P,T,F)$ be an arbitrary OEN. Using Proposition 3, we show $\mathcal{T}_F$ over the lattice $(\mathcal{T}_F, \subseteq)$ has a fixpoint. We refer to Tarski's fixpoint theorem \cite{tarski_fixpoint}, which states a monotonic operator on a complete lattice has a least fixpoint. We prove that $\mathcal{T}_F$ is a monotonic operator.

\begin{lemma}
    $\mathcal{T}_F$ is a monotonic operator over $\Omega$. For all $m_1, m_2 \in \Omega$, if $m_1 \subseteq m_2$ then $\mathcal{T}_F(m_1) \subseteq \mathcal{T}_F(m_2)$.
    \label{lemma:monotonicity}
\end{lemma}
\begin{proof}
    Assume $m_1 \subseteq m_2$ and we prove by enumerating possible cases. Take arbitrary $p \in P$. 
    
    Case 1: $p \in m_1$ and $p \in m_2$. By Definition 3, if arc $(t, p) \in F$ then $p \in \mathcal{T}_F(m_1), \mathcal{T}_F(m_2)$ otherwise $p \not\in \mathcal{T}_F(m_1), \mathcal{T}_F(m_2)$. 
    
    Case 2: $p \not \in m_1$, $t$ is enabled in $m_1$. Based on the assumption, places in $m_1$ with a token also have a token in $m_2$, $t$ must also be enabled in $m_2$. So $p \in \mathcal{T}_F(m_1), \mathcal{T}_F(m_2)$. 
    
    Case 3: $p \not \in m_1$, $t$ is not enabled in $m_1$ or $(t, p) \not \in F$. Then $p \not \in \mathcal{T}_F(m_1)$. So whether $t$ is enabled or not in $m_2$ does not affect $\mathcal{T}_F(m_1) \subseteq \mathcal{T}_F(m_2)$. 
    
    In all cases, $\mathcal{T}_F(m_1) \subseteq \mathcal{T}_F(m_2)$.
\end{proof}

We show that the least fixpoint of $\mathcal{T}_F$ is the union of reachable places by proving that $\mathcal{T}_F$ is a compact operator.

\begin{lemma}
    $\mathcal{T}_F$ is a compact operator over $\Omega$. For all $m_1 \in \Omega$, if $p \in \mathcal{T}_F(m_1)$ then $p \in \mathcal{T}_F(m_2)$ for a finite $m_2 \subseteq m_1$.
    \label{lemma:compactness}
\end{lemma}
\begin{proof}
    Assume $p \in \mathcal{T}_F(m_1)$. From Definition 3 and the assumption, a transition $t$ and a marking $m_1$ exist such that $\land_{(p_i, t)\in F} m_1(p_i)$. Take a marking $m_2$ that contains these places in $m_1$ so that $m_2$ = $\{ \, p_i \, | $ for arcs $(p_i, t), (t, p) \in F \} \subseteq m_1$. From Lemma 1, we know that a marking $m'$ exists since $m'(p) \gets \land_{(p_i, t)\in F} m_2(p_i)$ and $m' \subseteq \mathcal{T}_F(m_2)$, so $p \in \mathcal{T}_F(m_2)$.
\end{proof}

\smallskip\noindent\textbf{Proposition 4} The operator $\mathcal{T}_F$ has a least fixpoint equal to $\mathcal{T}_F{\uparrow^\alpha} = \cup_{n \geq 0}\mathcal{T}_F{\uparrow^n}$.
\begin{proof}
    $\mathcal{T}_F$ is a monotonic function. According to Tarski's fixpoint theorem \cite{tarski_fixpoint}, the function $\mathcal{T}_F$ has a least fixpoint $\mathcal{T}_F{\uparrow^\alpha}$ for some $\alpha \geq 0$. Since $\mathcal{T}_F$ is also a compact operator, if $p \in \mathcal{T}_F(\mathcal{T}_F{\uparrow^\alpha})$ then $p \in \mathcal{T}_F(\mathcal{T}_F{\uparrow^n})$ for some $n \geq 0$ and  $\mathcal{T}_F{\uparrow^n} \subseteq \mathcal{T}_F{\uparrow^\alpha}$. The union of subsets of $\mathcal{T}_F{\uparrow^\alpha}$ is $\cup_{n \geq 0}\mathcal{T}_F{\uparrow^n}$, which covers every member of the least fixpoint. Therefore, $\mathcal{T}_F{\uparrow^\alpha} = \cup_{n \geq 0}\mathcal{T}_F{\uparrow^n}$.
\end{proof}

Proposition 4 shows that when we repeatedly apply the step consequence operator to the empty set, the least fixpoint is the union of reachable places. This allows us to compute this least fixpoint by iteratively computing steps (fire transitions) and expanding the set of reachable places. A linear and immediately recursive program transformed from an OEN is logically equivalent to the OEN for this fixpoint semantics. 

\begin{theorem}
    Let $\mathcal{P}$ be a datalog program transformed from an OEN = (P, T, F). For constant symbols $p_1$, $p_2$ in $\mathcal{P}$, $\mathcal{P} \models r(p_1,p_2)$ if and only if $p_2$ is reachable from $p_1$ in the OEN.
    \label{theorem:appendix_equivalence}
\end{theorem}
\begin{proof}
    Prove the forward implication. Assume $\mathcal{P} \models r(p_1,p_2)$. We know that $r(p_1,p_2)$ is in the least model. Based on the definition of the Immediate Consequence Operator \cite{van_emden_semantics_1976}, $r(p_1,p_2) \gets \, B_1,..., B_i$ ($i \geq 0$) is a clause in $\mathcal{P}$ and $\{B_1,..., B_i \} \subseteq \mathcal{T_P}{\uparrow^\omega}$. From Definition 1 in the main paper, we know that $B_1,..., B_i$ are ground facts describing a set of transitions. There exists $n \geq 0$ such that $p_1 \in \mathcal{T}_F{\uparrow^n}$. Based on $B_1,..., B_i$, we can apply the operator $\mathcal{T}_F$ to $\mathcal{T}_F{\uparrow^n}$. From Definition 3 and proposition 4, we know $p_2 \in \mathcal{T}_F{\uparrow^\alpha}$ so there exists some marking $m \in \Omega$ such that $m(p_2) = \top$. Therefore, $p_2$ is reachable from $p_1$.  
    
    Now we prove the backward implication. Assume $p_2$ is reachable from $p_1$ in the OEN. Since both $p_1$ and $p_2$ are reachable from the initial marking, we know $p_1, p_2\in \mathcal{T}_F{\uparrow^\alpha}$. From Definition 3 and Proposition 4, there exists a finite sequence of steps (fired transitions) $m \xrightarrow[]{t_1} m_1 \xrightarrow[]{t_2} ... \xrightarrow[]{t_{k}} m'$ and $m(p_1) = m'(p_2) = \top$. From Definition 1 in the main paper, the transformed program contains these transitions as ground facts $\{B_1,..., B_k\} \subseteq \mathcal{B_P}$ ($k \geq 0$) and recursive clauses $r(X,Y) \gets r_1(X,Y)$ and $r(X,Y) \gets r_1(X,Z), r(Z,Y)$. We can ground the recursive clauses such that $r(p_1,p_2) \gets B_1,..., B_k$. By applying the $\mathcal{T_P}$ operator, we know $r(p_1,p_2) \in \mathcal{T_P}{\uparrow^\omega}$ so $\mathcal{P} \models r(p_1,p_2)$. 
\end{proof}
Theorem \ref{theorem:appendix_equivalence} shows the correctness of identifying reachable places in an OEN by evaluating the transformed linear and immediately recursive program. In our BMLP framework, we leverage low-level boolean matrix operations to evaluate linear and immediately recursive programs, which improves runtime efficiency in computing reachability in OENs.  
\section{Boolean matrix compilation}
\label{sec:appendix_matrices}
Our BMLP framework evaluates a linear and immediately recursive program by first compiling it into boolean matrices represented as datalog. We show examples of compiled boolean matrices based on the following linear and immediately recursive program. 
\begin{gather}
\mathcal{P}_1:\left\{
\begin{aligned}
    flight(c_1,c_2).  \quad flight(c_2,c_3).&\\
    route(X,Y) \gets flight(X,Y). &\\ 
    route(X,Y) \gets flight(X,Z), route(Z,Y). &\nonumber
\end{aligned}
\right\}
\end{gather}
This program has three constant symbols $c_1, c_2, c_3$ and two ground facts $flight(c_1,c_2), flight(c_2,c_3)$. So the number of constant symbols is $n = 3$ and the number of ground facts is $k = 2$. In Section 5 of the main paper, we propose two BMLP algorithms, iterative extension (BMLP-IE) and repeated matrix squaring (BMLP-RMS) for evaluating linear and immediately recursive programs. Both algorithms are implemented in SWI-Prolog and write boolean matrices as Prolog program files to be loaded and computed. These files contain binary codes that are written as integers. Each binary code corresponds to a set of constant symbols. Since BMLP-IE and BMLP-RMS use boolean matrices of different sizes as input, we show two examples of compiled boolean matrices. 

\subsection{Iterative extension (BMLP-IE)}
BMLP-IE takes two $k \times n$ boolean matrices as input. We will use $flight1$ and $flight2$ to refer to them. The first argument of $flight$ is mapped to $flight1$ and the second argument is mapped to $flight2$. The first ground fact $flight(c_1,c_2)$ in $\mathcal{P}_1$ is represented by the first row of boolean matrices (ordered from 0). Each bit in the binary code represents a constant symbol. The binary code 001 (written as 1) in $flight1(0,1)$ indicates that $c_1$ is the first argument in $flight(c_1,c_2)$. The binary code 010 (written as 2) in $flight2(0,2)$ denotes that $c_2$ is the second argument of $flight(c_1,c_2)$.
\begin{verbatim}
flight1(0,1). 
flight2(0,2).
\end{verbatim}
The second row of boolean matrices represents the second ground fact $flight(c_2,c_3)$.
\begin{verbatim}
flight1(1,2). 
flight2(1,4). 
\end{verbatim}
The binary code 100 (written as 4) in $flight2(1,4)$ means that $c_3$ is the second argument of $flight(c_2,c_3)$. Conceptually, this compilation splits every ground fact in the initial program into two ground facts in the boolean matrices.

BMLP-IE additionally takes a $1 \times n$ vector input to represent the constant symbol in a partially grounded clause, e.g. $flight(c_1,X)$. We write this vector as a binary code. In this case, the binary code's value is 001, which means that $c_1$ partially grounds the clause. 

\subsection{Repeated matrix squaring (BMLP-RMS)}
Only a $n \times n$ boolean matrix is given as an input to BMLP-RMS. We refer to this matrix as $flight1$. If the i-th constant symbol is related to the j-th constant symbol, the j-th bit of the i-th binary code will be 1. Since $c_1$ is related to $c_2$ in $flight(c_1,c_2)$, the second bit is 1 in the first row of the boolean matrix. The binary code is 010 (written as 2) in $flight1(0,2)$ (ordered from 0). $c_2$ is related to $c_3$ in $flight(c_2,c_3)$, so the third bit is 1 in the second row of the boolean matrix. The binary code is 100 (written as 4) in $flight1(1,4)$.  
\begin{verbatim}
flight1(0,2). 
flight1(1,4).
flight1(2,0).
\end{verbatim}
Finally, the third row is 0 because the third constant symbol $c_3$ is not related to any constant symbol.  We can consider a constant symbol as a node in a graph. The flight relations are the edges between nodes. The compiled $n \times n$ matrix is the incidence matrix, which describes the paths between nodes.
\section{Experiment materials}
\label{sec:appendix_exp}
For Prolog systems, we enable tabling to ensure termination and faster grounding re-evaluations. No additional clauses are needed for Clingo since Clingo computes the least Herbrand models by grounding and searching for stable models \cite{gelfond_stable_1988}. 

\subsection{Computing reachability in airplane flight routes} 
We generate artificial OENs by creating transitions from randomly sampled pairs of cities. OENs are transformed into linear and immediately recursive programs by rewriting OEN's places as constant symbols and transitions as predicates:
\begin{verbatim}
city(c1).
city(c18).
flight(c1,c18). 
...
route(X,Y):-flight(X,Y).  
route(X,Y):-flight(X,Z), route(Z,Y).
\end{verbatim}
In \textbf{sub-tasks 1}, we evaluate BMLP-IE's runtime in identifying reachable places from the city $c_1$ (the initial marking). We use B-Prolog, SWI-Prolog, XSB-Prolog and Clingo to enumerate groundings of $route(c1, Y)$. For Prolog systems, we introduce partially grounded clauses to enumerate provable ground facts:  
    \begin{verbatim}
closure:-route(c1,_C), fail.
closure.
\end{verbatim}
We compile $route(c1, Y)$ into a vector for BMLP-IE to compute its transitive closure. In \textbf{sub-tasks 2}, we assess BMLP-RMS's runtime for computing reachable places from all cities. We use B-Prolog, SWI-Prolog, XSB-Prolog and Clingo to find the least model of the datalog program. We write two clauses for Prolog systems to enumerate all provable ground facts:
\begin{verbatim}
closure:-route(_C1,_C2), fail.
closure.
\end{verbatim}
We compile all ground facts into a square boolean matrix for BMLP-RMS to return its transitive closure. 

\subsection{Computing producible substrates in metabolic network models} 
We represent the genome-scale metabolic network model iML1515 \cite{iML1515} as an OEN. Its transitions are reactions and places are sets of substrates involved in reactions. A set of related reactions is often referred to as a metabolic pathway. We then transform the OEN into a linear and immediately recursive program:  
\begin{verbatim}
substrates(ctp_c_h2o_c). 
reaction(ctp_c_h2o_c,cdp_c_h_c_pi_c).
...
metabolic_path(X,Y):-reaction(X,Y). 
metabolic_path(X,Y):-reaction(X,Z),
                     metabolic_path(Z,Y).
\end{verbatim}
From an initial marking of substrates, we use BMLP-IE to find substrates producible by the reactions in iML1515. 
\begin{verbatim}
closure:-metabolic_path(initial,_C),
         fail.
closure.
reaction(initial,ctp_c_h2o_c).
\end{verbatim}
We write two clauses for Prolog systems to enumerate all substrates producible by reactions. In this case, the initial marking assigns a token to the place ``ctp\_c\_h2o\_c''.

\section{Experiment results}
\label{sec:appendix_results}
In Table \ref{tab:exp_1_1}, \ref{tab:exp_1_2} and \ref{tab:exp_2} (at the end of the document), we report the mean and standard deviation of CPU time in addition to Section 6 of the main paper. We compare BMLP-IE and BMLP-RMS against B-Prolog, SWI-Prolog, XSB-Prolog and Clingo in tasks from two domains, graph and biological data. 

\newpage
\begin{table*}[ht]
    \centering
    \begin{tabular}{lrrrrr}
        \toprule
        Method & 1000 & 2000 & 3000 & 4000 & 5000 \\
        \midrule
        B-Prolog + tabling      & $0.001 \pm 0.001$ & $0.574 \pm 0.379$ & $4.499 \pm 0.147$ & $12.996 \pm 0.310$ & $28.104 \pm 2.700$        \\
        Clingo                  & $0.042 \pm 0.018$ & $9.806 \pm 1.135$ & $44.872 \pm 2.402$ & $125.917 \pm 3.397$ & $254.750 \pm 8.047$         \\
        \textbf{BMLP-IE$^\dagger$}    & $\textbf{0.007} \pm \textbf{0.004}$ & $\textbf{0.098} \pm \textbf{0.032}$ & $\textbf{0.135} \pm \textbf{0.049}$ & $\textbf{0.196} \pm \textbf{0.018}$ & $\textbf{0.297} \pm \textbf{0.034}$        \\
        SWI-Prolog + tabling    & $0.001 \pm 0.000$ & $2.781 \pm 1.888$ & $20.699 \pm 0.979$ & $51.906 \pm 1.004$ & $100.211 \pm 4.328$        \\
        XSB-Prolog + tabling    & $0.001 \pm 0.000 $ & $0.341 \pm 0.116$ & $2.075 \pm 0.072$ & $5.869 \pm 0.085$ & $12.213 \pm 0.287$  \\
        \bottomrule
    \end{tabular}
    \caption{Runtime for computing reachable places from one city. The probability $p_t$ of transitions (ground facts) is 0.001 but the No. places (constant symbols) $n$ varies in steps \{1000, 2000, 3000, 4000, 5000\}. BMLP-IE has a much lower runtime than the other methods for $n$ from 2000 to 5000. $^\dagger$BMLP-IE is implemented in SWI-Prolog. $^*$Prolog systems have tabling enabled. Since B-Prolog and XSB-Prolog can only print small decimal numbers as 0 for the CPU time statistics, we set the runtime lower bound to 0.001 seconds. }
    \label{tab:exp_1_1}
\end{table*}
\begin{table*}[ht]
    \centering
    \begin{tabular}{lrrrrrr}
        \toprule
        Method & 0.0001 & 0.001 & 0.01 & 0.1 & 0.5 & 1 \\
        \midrule
        B-Prolog$^*$ &  $0.001 \pm 0.000$ & $0.003 \pm 0.003$ & $2.042 \pm 0.036$ & $15.132 \pm 0.522$ & $74.635 \pm 1.173$ & $153.576 \pm 5.039$ \\
        Clingo & $0.004 \pm 0.001$  & $0.031 \pm 0.009$ & $5.660 \pm 0.059$ & $32.359 \pm 0.417$ & $236.174 \pm 1.384$ & timeout \\
        \textbf{BMLP-RMS$^\dagger$}    & $\textbf{1.561} \pm \textbf{0.297}$ & $\textbf{3.894} \pm \textbf{0.421}$ & $\textbf{6.140} \pm \textbf{0.172}$ & $\textbf{5.485} \pm \textbf{0.328}$ & $\textbf{2.871} \pm \textbf{0.126}$ & $\textbf{1.979} \pm \textbf{0.080}$ \\
        SWI-Prolog$^*$ & $0.001 \pm 0.000$  & $0.001 \pm 0.001$ & $7.516 \pm 0.267$ & $70.747 \pm 1.382$ & timeout & timeout \\
        XSB-Prolog$^*$ & $0.001 \pm 0.000$ & $0.003 \pm 0.001$ & $1.938 \pm 0.047$ & 1$5.974 \pm 1.585$ & $65.501 \pm 2.528$ & $109.778 \pm 4.230$\\
        \bottomrule
    \end{tabular}
    \caption{Runtime for computing reachable places from all cities. No. places (constant symbols) $n$ = 1000 and the probability $p_t$ of transitions (ground facts) varies in steps \{0.0001, 0.001, 0.01, 0.1, 0.5, 1\}. BMLP-RMS has a much lower runtime than the other methods for $p_t$ from 0.1 to 1. $^\dagger$BMLP-RMS is implemented in SWI-Prolog. $^*$Prolog systems have tabling enabled. Since B-Prolog and XSB-Prolog can only print small decimal numbers as 0 for the CPU time statistics, we set the runtime lower bound to 0.001 seconds. }
    \label{tab:exp_1_2}
\end{table*}
\begin{table*}[ht]
    \centering
    \begin{tabular}{lrrrrr}
        \toprule
        Method & 1 & 10 & 100 & 1000 \\
        \midrule
        B-Prolog$^*$      & $0.002 \pm 0.001$ & $0.003 \pm 0.001$ & $0.016 \pm 0.004$ & $1.193 \pm 0.374$\\
        Clingo                  &  $1.853 \pm 0.067$ & $2.054 \pm 0.172$ & $4.993 \pm 0.670$ & $61.121 \pm 11.399$     \\
        \textbf{BMLP-IE$^\dagger$}    & $\textbf{0.029} \pm \textbf{0.005}$ & $\textbf{0.042} \pm \textbf{0.017}$ & $\textbf{0.066} \pm \textbf{0.013}$ & $\textbf{0.095} \pm \textbf{0.011}$\\
        SWI-Prolog$^*$  & $0.012 \pm 0.001$ & $0.015 \pm 0.003$ & $0.091 \pm 0.021$ & $10.907 \pm 3.411$\\
        XSB-Prolog$^*$ & $0.001 \pm 0.000$ & $0.001 \pm 0.001$ & $0.009 \pm 0.002$ & $1.015 \pm 0.307$\\
        \bottomrule
    \end{tabular}
    \caption{Runtime for computing producible substrates in the metabolic network model iML1515. We test different initial markings by randomly sampling 1, 10, 100, and 1000 from 1877 substrates. BMLP-IE has the lowest runtime for the most challenging task when the number of substrates in the initial marking is 1000. $^\dagger$BMLP-IE is implemented in SWI-Prolog. $^*$Prolog systems have tabling enabled. }
    \label{tab:exp_2}
\end{table*}

\end{document}